\documentclass[11pt]{article} %
\usepackage{fullpage}
\usepackage{hyperref}
\usepackage{url}
\usepackage{graphicx}

\usepackage[round]{natbib}%
\usepackage{multirow,rotating}

\usepackage{amsthm}
\usepackage{amsmath,amssymb,latexsym,cite}
\usepackage[ruled,noend,noline,nofillcomment]{algorithm2e}
\usepackage{algorithmic}
\usepackage{stmaryrd}
\usepackage{enumerate}
\usepackage{enumitem}
\usepackage{tikz}
\usetikzlibrary{arrows}
\usetikzlibrary{calc}
\renewenvironment{proof}[1][\unskip]{   %
	\par\noindent{\em Proof #1.\ }}{\hfill\qed\par}

\newcommand{\cA}{\mathcal{A}}
\newcommand{\cB}{\mathcal{B}}

\newcommand{\cD}{\mathcal{D}}

\newcommand{\cF}{\mathcal{F}}

\newcommand{\cH}{\mathcal{H}}

\newcommand{\cX}{\mathcal{X}}
\newcommand{\cY}{\mathcal{Y}}

\newtheorem{theorem}{Theorem}[section]
\newtheorem{lemma}[theorem]{Lemma}
\newtheorem{cor}[theorem]{Corollary}

\newtheorem{fact}{Fact}
\renewcommand{\eqref}[1]{Eq.~(\ref{#1})}

\newcommand{\tabref}[1]{Table~\ref{#1}}        

\newcommand{\thmref}[1]{Theorem~\ref{#1}}
\newcommand{\lemref}[1]{Lemma~\ref{#1}}

\newcommand{\appref}[1]{Appendix~\ref{#1}}

\newcommand{\algref}[1]{Alg.~\ref{#1}}
\newcommand{\factref}[1]{Fact~\ref{#1}}

\def\mbf{\mathbf}

\def\argmin{\mathop{\rm argmin}}
\def\argmax{\mathop{\rm argmax}}

\renewcommand{\P}{\mathbb{P}}
\newcommand{\ap}[1]{\llbracket{#1}\rrbracket}

\newcommand{\binlab}{\{-1,+1\}}
\newcommand{\err}{\mathrm{err}}
\newcommand{\E}{\mathbb{E}}
\newcommand{\Var}{\textrm{Var}}

\newcommand{\one}{\mathbb{I}}

\newcommand{\vb}{\mathbf{b}}
\newcommand{\va}{\mathbf{a}}
\def\emax{\eta_{\mathrm{max}}}
\def\emin{\eta_{\mathrm{min}}}

\newcommand{\reals}{\mathbb{R}}

\newcommand{\half}{{\frac12}}
\newcommand{\ceil}[1]{{\lceil #1\rceil}}
\newcommand{\floor}[1]{\lfloor #1\rfloor}

\newcommand{\cht}{\cH_{\dashv}}
\newcommand{\chr}{\cH_{\Box}}
\newcommand{\errneg}{\err_{\mathrm{neg}}}

\newcommand{\OPT}{\mathrm{OPT}}

\newcommand{\caud}{\mathrm{OPT}_{\mathrm{cost}}}

\newcommand{\cost}{\mathrm{cost}}

\title{Auditing: Active Learning with Outcome-Dependent Query Costs}
\author{Sivan Sabato\thanks{Microsoft Research New England, Cambridge, MA USA, \texttt{sivan.sabato@microsoft.com}} \qquad Anand Sarwate\thanks{Toyota Technological Institute at Chicago, Chicago, IL USA, \texttt{asarwate@ttic.edu}} \qquad Nathan Srebro\thanks{Toyota Technological Institute at Chicago, Chicago, IL USA, \texttt{nati@ttic.edu}}}
\date{\today}

\begin{document}
\maketitle

%
\begin{abstract}
  We propose a learning setting in which unlabeled data is free, and
  the cost of a label depends on its value, which is not known in
  advance. We study binary classification in an extreme case, where
  the algorithm only pays for negative labels. Our motivation are
  applications such as fraud detection, in which investigating an
  honest transaction should be avoided if possible. We term the
  setting \emph{auditing}, and consider the \emph{auditing complexity}
  of an algorithm: the number of negative labels the algorithm
  requires in order to learn a hypothesis with low relative error. We
  design auditing algorithms for simple hypothesis classes (thresholds
  and rectangles), and show that with these algorithms, the auditing
  complexity can be significantly lower than the active label
  complexity.  We also discuss a general competitive approach for
  auditing and possible modifications to the framework.
\end{abstract}

\section{Introduction}
Active learning algorithms seek to mitigate the cost of learning by using unlabeled data and sequentially selecting examples to query for their label to minimize total number of queries.
In some cases, however, the actual cost of each query depends on the true label of the example and is thus {\em not known} before the label is requested. For instance, in detecting fraudulent credit transactions, a query with a positive answer is not wasteful, whereas a negative answer is the result of a wasteful investigation of an honest transaction, and perhaps a loss of good-will.
More generally, in a multiclass setting, different queries may entail different costs, depending on the outcome of the query.
In this work we focus on the binary case, and on the extreme version of the problem, as described in the example of credit frauds, in which the algorithm only pays for queries which return a negative label. We term this setting \emph{auditing}, and the cost incurred by the algorithm its \emph{auditing complexity}.

There are several natural ways to measure performance for auditing.
For example, we may wish for the algorithm to maximize the number of
positive labels it finds for a fixed ``budget'' of negative labels, or
to minimize the number of negative labels while finding a certain
number or fraction of positive labels. In this work we focus on the
classical learning problem, in which one attempts to learn a
classifier from a fixed hypothesis class, with an error close to the
best possible.  Similar to active learning, we assume we are given a
large set of unlabeled examples, and aim to learn with minimal
labeling cost.  But unlike active learning, we only incur a cost when
requesting the label of an example that turns out to be negative.

The close relationship between auditing and active learning raises
natural questions.  Can the auditing complexity be significantly
better than the label complexity in active learning?  If so, should
algorithms be optimized for auditing, or do optimal active learning
algorithms also have low auditing complexity?  To answer these
questions, and demonstrate the differences between active learning and
auditing, we study the simple hypothesis classes of thresholds and of axis-aligned
rectangles in $\reals^d$, in both the
realizable and the agnostic settings.  We then also consider a general
competitive analysis for arbitrary hypothesis classes.

Existing work on active learning with costs~\citep{Margineantu:07active,KapoorHB:07,Settles:08annotation,GolovinKr11} typically assumes that the cost of labeling each point is known \textit{a priori}, so the algorithm can use the costs directly to select a query.  Our model is significantly different, as the costs depend on the outcome of the query itself. 
\citet{KapoorHB:07} do mention the possibility of class-dependent costs, but this possibility is not studied in detail. %
An unrelated game-theoretic learning model addressing ``auditing'' was proposed by~\citet{BlockiCDS:11}.

\subsubsection*{Notation and Setup}

For an integer $m$, let $[m] = \{1,2,\ldots, m\}$.  
The function $\one[A]$ is the indicator function of a set $A$. For a function $f$ and a sub-domain $X$, $f|_X$ is the restriction of $f$ to $X$. For vectors $\mbf{a}$ and $\mbf{b}$ in $\mathbb{R}^d$, the inequality $\mbf{a} \le \mbf{b}$ implies $a_i \le b_i$ for all $i \in [d]$. 

We assume a data domain $\cX$ and a distribution $D$ over labeled data points in $\cX \times \binlab$. A learning algorithm may sample i.i.d.\ pairs $(X,Y)\sim D$. It then has access to the value of $X$, but the label $Y$ remains hidden until queried. The algorithm returns a labeling function $\hat{h}:\cX \rightarrow \binlab$. 
The error of a function $h:\cX \rightarrow \binlab$ on $D$ is $\err(D,h) = \E_{(X,Y) \sim D}[h(X) \neq Y].$
 The error of $h$ on a multiset $S \subseteq \cX \times \binlab$ is given by $\err(S,h) = \frac{1}{|S|} \sum_{(x,y)\in S} \one[ h(x) \ne y ].$
	The \emph{passive sample complexity} of an algorithm is the number of pairs it draws from $D$.
	The \emph{active label complexity} of an algorithm is the total number of label queries the algorithm makes. Its \emph{auditing complexity} is the number of queries the algorithm makes on points with negative labels. 
	
We consider guarantees for learning algorithms relative to a hypothesis class $\cH \subseteq \binlab^\cX$.  We denote the error of the best hypothesis in $\cH$ on $D$ by $\err(D,\cH) = \min_{h \in \cH}\err(D,h)$. Similarly, $\err(S,\cH) = \min_{h\in\cH}\err(S,h)$. We usually denote the best error for $D$ by $\eta = \err(D,\cH)$.

To describe our algorithms it will be convenient to define the following sample sizes, using universal constants $C,c >0$.  Let $\delta \in (0,1)$ be a confidence parameter, and let $\epsilon \in (0,1)$ be an error parameter. Let $m^{\mathrm{ag}}(\epsilon,\delta,d) = C(d+\ln(c/\delta))/\epsilon^2$.  If a sample $S$ is drawn from $D$ with $|S|=m^{\mathrm{ag}}(\epsilon,\delta,d)$ then with probability $1-\delta$, $\forall h \in \cH, \err(D,h) \leq \err(S,h)+\epsilon$ and $\err(S,\cH) \leq \err(D,\cH) + \epsilon$ \citep{BartlettMe02}.   Let $m_\nu(\epsilon, \delta, d) = C(d\ln(c/\nu\epsilon)+\ln(c/\delta))/\nu^2\epsilon$. Results of \citet{VapnikCh71} show that if $\cH$ has VC dimension $d$ and $S$ is drawn from $D$ with $|S| = m_{\nu}$, then for all $h\in\cH$,
\begin{align}\label{eq:vc}
\err(S,h) &\leq \max\left\{ \err(D,h)(1+\nu), \err(D,h)+ \nu \epsilon \right\}  \text{ and }\\
\err(D,h) &\leq \max\left\{ \err(S,h)(1+\nu), \err(S,h) + \nu \epsilon \right\}.\notag
\end{align}
\label{page:vc}%

\section{Active Learning vs. Auditing: Summary of Results}

The main point of this paper is that the auditing complexity
can be quite different from the active label complexity, and that
algorithms tuned to minimizing the audit label complexity give improvements
over standard active learning algorithms.
Before presenting these differences, we note that in some regimes,
neither active learning nor auditing can improve significantly over
the passive sample complexity.  In particular, a simple adaptation of
a result of
\citet{BeygelzimerDaLa09} establishes the following lower bound.
\begin{lemma}
\label{lem:vclower}
Let $\cH$ be a hypothesis class with VC dimension $d > 1$.  If an algorithm always finds a
hypothesis $\hat{h}$ with $\err(D,\hat{h}) \le \err(D,\cH) +
\epsilon$ for $\epsilon > 0$, then for any $\eta \in (0,1)$ there is a distribution $D$ with $\eta = \err(D,\cH)$
such that the auditing complexity of this algorithm for $D$ is $\Omega(d\eta^2/\epsilon^2)$.
\end{lemma}
That is, when $\eta$ is fixed while $\epsilon\rightarrow 0$, the
auditing complexity scales as $\Omega(d/\epsilon^2)$, similar to the
passive sample complexity.  Therefore the two situations which are interesting
are the realizable case, corresponding to $\eta=0$, and the agnostic case, when we want to
guarantee an excess error $\epsilon$ such that $\eta/\epsilon$ is
bounded.  We provide results for both of these regimes.  

We will first consider the \emph{realizable case}, when $\eta = 0$.  Here it is sufficient to consider the case where a fixed pool $S$ of $m$ points is given and the algorithm must return a hypothesis $\hat{h}$ such that $\err(S,\hat{h}) = 0$ with probability $1$.  A pool labeling algorithm can be used to learn a hypothesis which is good for a distribution by drawing and labeling a large enough pool. We define auditing complexity for an unlabeled pool as the minimal number of negative labels needed to perfectly classify it.  It is easy to see that there are pools with an auditing complexity at least the VC dimension of the hypothesis class. 

For the \emph{agnostic case}, when $\eta>0$, we denote $\alpha=\epsilon/\eta$ and
say that an algorithm $(\alpha,\delta)$-learns a class of
distributions $\cD$ with respect to $\cH$ if for all $D \in \cD$, with
probability $1-\delta$, $\hat{h}$ returned by the algorithm satisfies
$\err(D,\hat{h}) \le (1 + \alpha) \eta$.  By Lemma \ref{lem:vclower} an
auditing complexity of $\Omega(d/\alpha^2)$ is unavoidable, be we can
hope to improve over the passive sample complexity lower bound of
$\Omega(d/\eta\alpha^2)$ \citep{DevroyeLu95} by avoiding the dependence on $\eta$.

Our main results are summarized in \tabref{tab:main}, which shows
the auditing and active learning complexities in the two regimes,
for thresholds on $[0,1]$ and axis-aligned rectangles in $\reals^d$, 
where we assume that the hypotheses label the points in the rectangle as negative and points outside as positive.

\begin{table}[hb]
\centering
\begin{tabular}{p{2cm}l|c|c|c}
& &  \textbf{Active} & \textbf{Auditing} \\
\hline
\multirow{2}{*}{\vbox{\centering \textbf{\small Realizable}}}
	& Thresholds & $\Theta(\ln m)$ & $1$ \\
	& Rectangles & $m$ & $2d$ \\
\hline
\multirow{2}{*}{\vbox{\centering \textbf{\small Agnostic}  }}
	& Thresholds
	& $\Omega\left(\ln\left(\frac{1}{\eta}\right)+\frac{1}{\alpha^2}\right)$ 
	& $O\left( \frac{1}{\alpha^2} \right)$ \\	
	& Rectangles  
	&  $\Omega\left(d \left( \frac{1}{\eta} + \frac{1}{\alpha^2} \right) \right)$
	&  $O\left(d^2 \ln^2\left(\frac{1}{\eta}\right)\cdot \frac{1}{\alpha^2} \ln\left( \frac{1}{\alpha} \right) \right)$\\
\end{tabular}
\caption{Auditing complexity upper bounds vs. active label complexity lower bounds for realizable (pool size $m$) and agnostic ($\err(D,\cH) = \eta$) cases.  Agnostic bounds are for $(\alpha,\delta)$-learning with a fixed $\delta$, where $\alpha = \epsilon/\eta$.  \label{tab:main}
}
\end{table}

In the realizable case, for thresholds, the optimal active learning
algorithm performs binary search, resulting in $\Omega(\ln m)$ labels
in the worst case. This is a significant improvement over the passive
label complexity of $m$.  However, a simple auditing procedure that
scans from right to left queries only a single negative point,
achieving an auditing complexity of $1$. For rectangles, we
present a simple coordinate-wise scanning procedure with auditing complexity
of at most $2d$, demonstrating a huge gap versus active learning,
where the labels of all $m$ points might be required.  Not all classes
enjoy reduced auditing complexity: we also show that for rectangles
with positive points on the \emph{inside},
there exists pools of size $m$ with an auditing complexity of $m$.

In the agnostic case we wish to $(\alpha,\delta)$-learn distributions with a true error of $\eta = \err(D,\cH)$, for constant $\alpha,\delta$.
For active learning, it has been shown that in some cases, the $\Omega(d/\eta)$ passive sample complexity
can be replaced by an exponentially smaller $O(d\ln(1/\eta))$ active label complexity \citep{Hanneke11}, albeit sometimes with a larger polynomial dependence on $d$. In other cases, an $\Omega(1/\eta)$ dependence exists also for active learning.
Our main question is whether the dependence on $\eta$ in the active label complexity can be further reduced for auditing.

For thresholds, active learning requires $\Omega(\ln(1/\eta))$ labels \citep{Kulkarni93}. Using auditing, 
we show that the dependence on $\eta$ can be completely removed, for any true error level $\eta > 0$, {\em if we know
  $\eta$ in advance}.  We also show that if $\eta$ is not known
at least approximately, the logarithmic dependence on $1/\eta$ is also
unavoidable for auditing.  For rectangles, we show that
the active label complexity is at least $\Omega(d/\eta)$. In contrast, we propose an algorithm with an
auditing complexity of $O(d^2\ln^2(1/\eta))$, reducing the linear dependence on $1/\eta$ to a logarithmic
dependence.  We do not know whether a linear dependence on $d$ is
possible with a logarithmic dependence on $1/\eta$.

Most of the proofs are provided in \appref{app:proofs}.

\section{Auditing for Thresholds on the Line \label{sec:thresholds}}

The first question to ask is whether the audit label complexity can ever be significantly smaller than the active or passive label complexities, and whether a different algorithm is required to achieve this improvement. The following simple case answers both questions in the affirmative.
Consider the hypothesis class of \emph{thresholds on the line}, defined over the domain $\cX= [0,1]$.  A hypothesis with threshold $a$ is $h_a(x) = \one[x-a \geq 0]$.
The hypothesis class is $\cht = \{h_a \mid a \in [0,1]\}$. Consider
the pool setting for the realizable case. The optimal active label
complexity of $\Theta(\log_2 m )$ can be achieved by a binary search
on the pool. The auditing complexity of this algorithm can also be as
large as $\Theta(\log_2(m))$. However, auditing allows us to beat this
barrier. 
This case exemplifies an interesting contrast between auditing and active learning. Due to information-theoretic considerations, any algorithm which learns an unlabeled pool $S$ has an active label complexity of at least $\log_2 |\cH|_{S}|$ \citep{Kulkarni93}, where $\cH|_{S}$ is the set of restrictions of functions in $\cH$ to the domain $S$. For $\cht$, the active label complexity is thus $\log_2 |\cht|_{S}| = \Omega(\log_2 m )$. However, the same considerations are invalid for auditing. 

We showed that for the realizable case, the auditing label complexity
for $\cht$ is a constant.  We now provide a more complex algorithm that
guarantees this for $(\alpha,\delta)$-learning in the agnostic case.  
The intuition behind our approach is that in a pool with at most $k$ errors, querying from highest to lowest until observing $k+1$ negative points, and finding the minimal error threshold on the labeled points, yields the optimal threshold.

\begin{lemma} \label{lem:line:pool}
Let $S$ be a pool of size $m$ in $[0,1]$, and assume that $\err(S,\cht) \leq k/m$. Then the procedure above finds $\hat{h}$ such that $\err(S,\hat{h}) = \err(S,\cht)$ with an auditing complexity of $k+1$.
\end{lemma}

\begin{proof}
Denote the last queried point by $x_0$, and let $h_{a^*} = \argmin_{h \in \cht}\err(S,\cht)$. Since $\err(S,h_{a^*}) \leq k/m$, $a^* > x_0$. Denote by $S' \subseteq S$ the set of points queried by the procedure. For any $a > x_0$,
$\err(S',h_a) = \err(S,h_a) + |\{(x,y) \in S\mid x < x_0, y = 1\}|/m.$ Therefore, minimizing the error on $S'$ results in a hypothesis that minimizes the error on $S$.
\end{proof}

To learn from a distribution, one can draw a random sample and use it as the pool in the procedure above. However, the sample size required for passive $(\alpha,\delta)$-learning of thresholds is $\Omega(\ln(1/\eta)/\eta)$. Thus, the number of errors in the pool would be $k = \eta\cdot \Omega(\ln(1/\eta)/\eta) = \Omega(\ln(1/\eta))$, which depends on $\eta$. To avoid this dependence, the auditing algorithm we propose uses \algref{alg:subset} below to select a subset of the random sample, which still represents the distribution well, but its size is only $\Omega(1/\eta)$.

\begin{algorithm}[ht]
 \begin{algorithmic}[1]
  \caption{Representative Subset Selection} \label{alg:subset}
 \STATE {\bf Input:} pool $S = (x_1,\ldots,x_m)$ (with hidden labels), $x_i \in [0,1]$, $\emax \in (0,1]$, $\delta \in (0,1)$.
 \STATE $T \leftarrow \max\{\floor{1/{3\emax}},1\}$.
 \STATE Let $U = \{ \underbrace{x_1, \ldots, x_1}_{T\ \text{copies}}, \ldots, \underbrace{x_m, \ldots, x_m}_{T\ \text{copies}}\}$ be the multiset with $T$ copies of each point in $S$. 
 \STATE Sort and rename the points in $U$ such that $x'_i \leq x'_{i+1}$ for all $i \in [Tm]$.
 \STATE Let $S_q$ be an empty multiset.
 \FOR {$t=1$ to $T$}
 \STATE $S(t) \leftarrow \{x'_{(t-1)m+1},\ldots,x'_{tm}\}$.
 \STATE Draw $14\ln(8/\delta)$ random points from $S(t)$ independently uniformly at random and add them to $S_q$ (with duplications).
 \ENDFOR
 \STATE Return $S_q$ (with the corresponding hidden labels).
\end{algorithmic}
\end{algorithm}

\begin{lemma}
\label{lem:line:subsample}
Let $\delta,\emax \in (0,1)$.  Let $S$ be a pool such that $\err(S,\cht) \le \emax$.  Let $S_q$ be the output of \algref{alg:subset} with inputs $S,\emax,\delta$, and let $\hat{h} = \argmin_{h \in \cht} \err(S_q,\cht)$.  Then with probability $1-\delta$,
\begin{align*}
	\err(S_q,\hat{h}) \leq 6 \emax \quad \text{ and }	\quad
	\err(S, \hat{h}) \leq 17 \emax.
\end{align*}
\end{lemma}

\begin{algorithm}[ht]
 \begin{algorithmic}[1]
  \caption{Auditing for Thresholds with a constant $\alpha$} \label{alg:thresh}
    \STATE {\bf Input:} $\emax,\delta,\alpha \in (0,1)$, access to distribution $D$ such that $\err(D,\cht) \leq \emax$.
  \STATE $\nu \leftarrow \alpha/5$.
     \STATE Draw a random labeled pool (with hidden labels) $S_0$ of size $m_{\nu}(\eta, \delta/2,1)$ from $D$.
 \STATE Draw a random sample $S$ of size $m^{\mathrm{ag}}((1+\nu)\emax, \delta/2,1)$ uniformly from $S_0$.
 \STATE Get a subset $S_q$ using \algref{alg:subset} with inputs $S,2(1+\nu)\emax,\delta/2$. 
 \STATE Query points in $S_q$ from highest to lowest. Stop after $\ceil{12 |S_q| (1+\nu)\emax} + 1$ negatives.\label{step:sq}
 \STATE Find $\hat{\va}$ such that $h_{\hat{\va}}$ minimizes the error on the labeled part of $S_q$.\label{step:hata}
 \STATE Let $S_1$ be the set of the $36(1+\nu)\emax|S_0|$ closest points to $\hat{\va}$ in $S$ from each side of $\hat{\va}$.\label{step:s1}
 \STATE Draw $S_2$ of size $m^{\mathrm{ag}}(\nu/72,\delta/2,1)$ from $S_1$ (see definition on page \pageref{page:vc}).
 \STATE Query all points in $S_2$, and return $\hat{h}$ that minimizes the error on $S_2$. 
 \end{algorithmic}
\end{algorithm}

The algorithm for auditing thresholds on the line in the agnostic case is listed in \algref{alg:thresh}.
This algorithm first achieves $(C,\delta)$ learning of $\cht$ for a fixed $C$ (in step \ref{step:hata}, based on \lemref{lem:line:subsample} and \lemref{lem:line:pool}, and then improves its accuracy to achieve $(\alpha,\delta)$-learning for $\alpha > 0$, by additional passive sampling in a restricted region. The following theorem provides the guarantees for \algref{alg:thresh}.
     
\begin{theorem}\label{thm:threshalpha}
Let $\emax,\delta,\alpha \in (0,1)$.
Let $D$ be a distribution with error $\err(D,\cht) \le \emax$. 
\algref{alg:thresh} with input $\emax,\delta,\alpha$ has an auditing complexity of $O(\ln(1/\delta)/\alpha^2)$, and  returns $\hat{h}$ such that with probability $1-\delta$, $\err(D,\hat{h}) \leq (1+\alpha)\emax$.
\end{theorem}
It immediately follows that if $\eta = \err(D,\cH)$ is known, $(\alpha,\delta)$-learning is achievable with an auditing complexity that does not depend on $\eta$. This is formulated in the following corollary.
\begin{cor}[$(\alpha,\delta)$-learning for $\cht$]
Let $\eta,\alpha,\delta \in (0,1]$. For any distribution $D$ with error $\err(D,\cht) = \eta$, \algref{alg:thresh} with inputs $\emax=\eta,\alpha,\delta$  $(\alpha,\delta)$-learns $D$ with respect to $\cht$ with an auditing complexity of $O(\ln(1/\delta)/\alpha^2)$.
\end{cor}
A similar result holds if the error is known up to a multiplicative constant. But what if no bound on $\eta$ is known? The following lower bound shows that in this case, the best active complexity for threshold this similar to the best active label complexity. 

\begin{theorem}[Lower bound on auditing $\cht$ without $\emax$]\label{thm:lowerthresh}
Consider any constant $\alpha \geq 0$. For any $\delta \in (0,1)$, if an auditing algorithm $(\alpha,\delta)$-learns any distribution $D$ such that $\err(D,\cht) \ge \emin$, then the algorithm's auditing complexity is  $\Omega(\ln(\frac{1-\delta}{\delta})\ln(1/\emin))$.
\label{thm:thresh:noemax}
\end{theorem}
Nonetheless, in the next section show that there are classes with a significant gap between active and auditing complexities even without an upper bound on the error.

\section{Axis Aligned Rectangles \label{sec:aprs}}

A natural extension of thresholds to higher dimension is the class of
axis-aligned rectangles, in which the labels are determined by a
$d$-dimensional hyperrectangle.  This hypothesis class, first
introduced in \citet{BlumerEhHaWa89}, has been studied extensively in
different regimes \citep{kearns1998efficient,long1998pac}, including
active learning \citep{hanneke2007teaching}.  An
axis-aligned-rectangle hypothesis is a disjunction of $2d$ thresholds.
For simplicity of presentation, we consider here the slightly simpler
class of disjunctions of $d$ thresholds over the positive orthant
$\mathbb{R}^{d}_+$. 
It is easy to reduce learning of an
axis-aligned rectangle in $\reals^d$ to learning of a disjunction of thresholds in $\reals^{2d}$, by mapping each point $\mbf{x} \in \reals^d$ to a point $\tilde{\mbf{x}} \in \reals^{2d}$ such that for $i \in [d]$, $\tilde{x}[i] = \max(x[i],0)$ and $\tilde{x}[i+d] = \max(0,-x[i]))$.
Thus learning the class of disjunctions is equivalent, up to a factor of two in the dimensionality, to learning rectangles.\footnote{This reduction suffices if the origin is known to be in the rectangle. Our algorithms and results can all be extended to the case where rectangles are not required to include the origin. To keep the algorithm and analysis as simple as possible, we state the result for this special case.} 
Because auditing costs are asymmetric, we consider two possibilities
for label assignment.  For a vector $\va = (a[1],\ldots,a[d]) \in
\reals_+^d$, define the hypotheses $h_\va$ and $h_\va^-$ by
	\begin{align*}
	h_\va(x) = 2 \one[\exists i\in [d], x[i] \geq a[i]] - 1,\quad\text{ and } \quad h_\va^{-}(x) = - h_\va(x).
	\end{align*}
Define $\chr = \{h_\va \mid \va \in \reals_+^d\}$ and $\chr^- = \{h^{-}_\va \mid \va \in \reals_+^d\}$. In $\chr$ the positive points are outside the rectangle and in $\chr^-$ the negatives are outside. Both classes have VC dimension $d$. All of our results for these classes can be easily extended to the corresponding classes of general axis-aligned rectangles on $\reals^d$, with at most a factor of two penalty on the auditing complexity.

\subsection{The Realizable Case}\label{sec:aprpool}

We first consider the pool setting for the realizable case, and show a sharp contrast between the auditing complexity and the active label complexity for $\chr$ and $\chr^-$. Assume a pool of size $m$.
While the active learning complexity for $\chr$ and $\chr^-$ can be as large as $m$,
the auditing complexities for the two classes are quite different. 
For $\chr^-$, the auditing complexity can be as large as $m$, but for $\chr$ it is at most $d$. 
We start by showing the upper bound for auditing of $\chr$.

\begin{theorem}[Pool auditing upper bound for $\chr$] \label{thm:aprnoiseless}
The auditing complexity of any unlabeled pool $S_u$ of size $m$ with respect to $\chr$ is at most $d$.
\end{theorem}

\begin{proof}
The method is a generalization of the approach to auditing for thresholds. Let $h^* \in \chr$ such that $\err(S,h^*) = 0$. For each $i \in [d]$, order the points $x$ in $S$ by the values of their $i$-th coordinates $x[i]$. Query the points sequentially from largest value to the smallest (breaking ties arbitrarily) and stop when the first negative label is returned, for some point $\mbf{x}_i$.  Set $a[i] \leftarrow x_i[i]$, and note that $h^*$ labels all points in $\{ \mbf{x} \mid x[i] > a[i] \}$ positive.  Return the hypothesis $\hat{h} = h_\va$. This procedure clearly queries at most $d$ negative points and agrees with the labeling of $h^*$.
\end{proof}

It is easy to see that for full Axis-Aligned Rectangles, an auditing complexity of $2d$ can be achieved in a similar fashion. We now show the lower bound for the auditing complexity of $\chr^-$. 
This immediately implies the same lower bound for active label complexity of $\chr^-$ and $\chr$.

\begin{theorem}[Pool auditing lower bound for $\chr^-$] \label{thm:aprlower}
For any $m$ and any $d \geq 2$, there is a pool $S_u \subseteq \mathbb{R}^{d}_+$ of size $m$ such that its auditing complexity with respect to $\chr^-$ is $m$.
\end{theorem}

\begin{proof}
The construction is a simple adaptation of a construction due to \citet{Dasgupta04}, originally showing an active learning lower bound for the class of hyperplanes. Let the pool be composed of $m$ distinct points on the intersection of the unit circle and the positive orthant: $S_u = \{ (\cos \theta_{j}, \sin \theta_{j})\}$ for distinct $\theta_{j} \in [0, \pi/2]$. Any labeling which labels all the points in $S_u$ negative except any one point is realizable for $\chr^-$, and so is the all-negative labeling. Thus, any algorithm that distinguishes between these different labelings with probability $1$ must query all the negative labels. 
\end{proof}

\begin{cor}[Realizable active label complexity of $\chr$ and $\chr^{-}$]
For $\chr$ and $\chr^-$, there is a pool of size $m$ such that its active label complexity is $m$.
\end{cor}

\subsection{The Agnostic Case}

We now consider $\chr$ in the agnostic case, where $\eta > 0$. The best known algorithm for active learning of rectangles $(2,\delta)$-learns a very restricted class of distributions (continuous product distributions which are sufficiently balanced in all directions) with an active label complexity of $\tilde{O}(d^3 p(\ln(1/\eta) p(\ln(1/\delta)))$, where $p(\cdot)$ is a polynomial \citep{hanneke2007teaching}. 
However, for a general distribution, active label complexity cannot be significantly better than passive label complexity. This is formalized in the following theorem.

\begin{theorem}[Agnostic active label complexity of $\chr$]
\label{thm:apr:ag:lower}
Let $\alpha,\eta > 0, \delta \in (0,\half)$. Any learning algorithm that $(\alpha,\delta)$-learns all distributions such that $\err(D,\cH) = \eta$ for $\eta > 0$ with respect to $\chr$ has an active label complexity of $\Omega(d/\eta)$.
\end{theorem}

In contrast, the auditing complexity of $\chr$ can be much smaller, as we show for \algref{alg:aprs} below.

\begin{algorithm}[ht]
 \begin{algorithmic}[1]
  \caption{Auditing for $\chr$} \label{alg:aprs}
 \STATE {\bf Input:} $\emin > 0$, $\alpha \in (0,1]$, access to distribution $D$ over $\reals_+^d \times \binlab$.
 \STATE $\nu \leftarrow \alpha/25$.
 \FOR {$t=0$ to $\floor{\log_2(1/\emin)}$}
 \STATE $\eta_t \leftarrow 2^{-t}$.
 \STATE Draw a sample $S_t$ of size $m_\nu(\eta_t,\delta/\log_2(1/\emin),10d)$ with hidden labels.
 \FOR{ $i= 1$ to $d$}\label{step:starti}
 \STATE $j \leftarrow 0$
 \WHILE{ $j \leq \ceil{(1+\nu)\eta_t|S_t|}+1$}
 \STATE If unqueried points exist, query the unqueried point with highest $i$'th coordinate;
 \STATE If query returned $-1$, $j \leftarrow j + 1$.
 \ENDWHILE
 \STATE $b_t[i] \leftarrow $ the $i$'th coordinate of the last queried point, or $0$ if all points were queried.
	\ENDFOR\label{step:endi}
 \STATE Set $S_{\vb_t}$ to $S_t$, with unqueried labels set to $-1$.
 \STATE $V_t \leftarrow V_{\nu}(S_{\vb_t}, \eta_t, \chr[\vb_t])$. \label{step:vt}
 \STATE $\hat{\eta}_t \leftarrow \max_{h \in V_t} \errneg(S_{\vb_t},h)$.
 \IF{ $\hat{\eta}_t > \eta_t/4$} \label{step:if} 
 \STATE Skip to step \ref{step:return}
 \ENDIF
 \ENDFOR
\STATE Return $\hat{h} \equiv \argmin_{h\in \chr[\vb_t]} \err(S_{\vb_t}, h)$.\label{step:return}
\end{algorithmic}
\end{algorithm}

\begin{theorem}[Auditing complexity of $\chr$]  \label{thm:aprmain}
For $\emin,\alpha,\delta \in (0,1)$, \algref{alg:aprs} $(\alpha,\delta)$-learns all distributions with $\eta \geq \emin$ with respect to $\chr$ with an auditing complexity of $O(\frac{d^2\ln(1/\alpha\delta)}{\alpha^2}\ln^2(1/\emin))$.
\end{theorem}

If $\emin$ is polynomially close to the true $\eta$, we get an auditing complexity of $O(d^2\ln^2(1/\eta))$,
compared to the active label complexity of $\Omega(d/\eta)$, an exponential improvement in $\eta$. It is an open question whether the quadratic dependence on $d$ is necessary here.

\algref{alg:aprs} implements a `low-confidence' version of the realizable algorithm. It sequentially queries points in each direction, until enough negative points have been observed to make sure the threshold in this direction has been overstepped. To bound the number of negative labels, the algorithm 
iteratively refines lower bounds on the locations of the best thresholds, and an upper bound on the \emph{negative error}, defined as the probability that a point from $D$ with negative label is classified as positive by a minimal-error classifier. The algorithm uses queries that mostly result in positive labels, and stops when the upper bound on the negative error cannot be refined.
The idea of iteratively refining a set of possible hypotheses has been used in a long line of active learning works \citep{CohnAtLa94,BalcanBeLa06,Hanneke07,DasguptaHM:2008}. Here we refine in a particular way that uses the structure of $\chr$, and allows bounding the number of negative examples we observe.

We use the following notation in \algref{alg:aprs}. The negative error of a hypothesis is $\errneg(D,h) = \P_{(X,Y) \sim D}[h(X) = 1 \text{ and } Y = -1]$.
It is easy to see that the same convergence guarantees that hold for $\err(\cdot,\cdot)$ using a sample size $m_\nu(\epsilon,\delta,d)$ hold also for the negative error $\errneg(\cdot,\cdot)$ (see \lemref{lem:errneg}). 
For a labeled set of points $S$, an $\epsilon \leq (0,1)$ and a hypothesis class $\cH$, 
denote 
$V_\nu(S,\epsilon, \cH) =  \{ h \in \cH \mid \err(S,h) \leq \err(S,\cH) + (2 \nu + \nu^2)\cdot\max(\err(S,\cH),\epsilon )\}$.
For a vector $\vb \in \reals_+^d$, define $\chr[\vb] = \{ h_\va \in \chr \mid \va \geq \vb\}$.

\thmref{thm:aprmain} is proven in Section \ref{thm:aprmain:pf}. The proof idea is to show that at each round $t$, $V_t$ includes any $h^* \in \argmin_{h \in \cH}\err(D,h)$, and $\hat{\eta}_t$ is an upper bound on $\errneg(D,h^*)$. Further, at any given point minimizing the error on $S_{\vb_t}$ is equivalent to minimizing the error on the entire (unlabeled) sample. We conclude that the algorithm obtains a good approximation of the total error. Its auditing complexity is bounded since it queries a bounded number of negative points at each round.

\section{Outcome-dependent Costs for a General Hypothesis Class \label{sec:general}}

In this section we return to the realizable pool setting and consider finite hypothesis classes $\cH$. We consider general outcome-dependent costs and a general space of labels $\cY$, so that $\cH\subseteq \cY^\cX$. Let $S \subseteq \cX$ be an unlabeled pool, and let $\cost:S\times \cH \rightarrow \reals_+$ denote the cost of a query: For $x\in S$ and $h \in \cH$, $\cost(x,h)$ is the cost of querying the label of $x$ given that $h$ is the true (unknown) hypothesis. In the auditing setting, $\cY = \{-1,+1\}$ and  $\cost(x,h) = \one[h(x) = -1]$. For active learning, $\cost \equiv 1$. Note that under this definition of cost function, the algorithm may not know the cost of the query until it reveals the true hypothesis.

Define $\caud(S)$ to be the minimal cost of an algorithm that for any labeling of $S$ which is consistent with some $h \in \cH$ produces a hypothesis $\hat{h}$ such that $\err(S,\hat{h}) = 0$. 
In the active learning setting, where $\cost \equiv 1$, it is NP-hard to obtain $\caud(S)$ for general $\cH$ and $S$. This can be shown by a reduction to set-cover \citep{HyafilRi76}. A simple adaptation of the reduction for the auditing complexity, which we defer to the full version of this work, shows that it is also NP-hard to obtain $\caud(S)$ in the auditing setting.

For active learning, and for query costs that do not depend on the true hypothesis (that is $\cost(x,h) \equiv \cost(x)$), \citet{GolovinKr11} showed an efficient greedy strategy that achieves a cost of $O(\caud(S)\cdot\ln(|\cH|))$ for any $S$. This approach has also been shown to provide considerable performance gains in practical settings \citep{GonenSaSh13}. The greedy strategy consists of iteratively selecting a point whose label splits the set of possible hypotheses as evenly as possible, with a normalization proportional on the cost of each query.

We now show that for outcome-dependent costs, if there are two labels and the cost depends only on the label, then another greedy strategy provides similar approximation guarantees for $\caud(S)$. The algorithm is defined as follows: Suppose that so far the algorithm requested labels for $x_1,\ldots,x_t$ and received the corresponding labels $y_1,\ldots,y_t$. Letting $S_t = \{(x_1,y_1),\ldots,(x_t,y_t)\}$, denote the current version space by $V(S_t) =\{h \in \cH|_S \mid \forall (x,y) \in S_t , h(x) = y\}$. The next query selected by the algorithm is 
\[
x \in \argmax_{x \in S} \min_{h \in \cH} \frac{|V(S_t)\setminus V(S_t \cup \{(x,h(x))\})|}{\cost(x,h)}.
\]
That is, the algorithm selects the query that in the worst-case over the possible hypotheses, would remove the most hypotheses from the version spaces, when normalizing by the outcome-dependent cost of the query. 
The algorithm terminates when $|V(S_t)| = 1$, and returns the single hypothesis in the version space.

\begin{theorem}\label{thm:greedy}
For any hypothesis class $\cH$ with $|\cY| = 2$, any pool $S$, and any true hypothesis $h \in \cH$, if $\cost(x,h) \equiv \cost(x,h(x))$, then the cost of the proposed algorithm is at most $(\ln(|\cH|_S|-1)+1)\cdot \OPT$.\footnote{A stronger version was erroneously given in the short version of this paper. However, our proof holds only for this weaker version.}
\end{theorem}

If $\cost$ is the auditing cost, the proposed algorithm is mapped to the following intuitive strategy: At every round, select a query such that, if its result is a negative label, then the number of hypotheses removed from the version space is the largest. 
This strategy is consistent with a simple principle based on a partial ordering of the points: 
For points $x,x'$ in the pool, define $x' \preceq x$ if $\{h \in \cH \mid h(x') = -1\} \supseteq \{h \in \cH \mid h(x) = -1\}$, so that if $x'$ has a negative label, so does $x$. In the auditing setting, it is always preferable to query $x$ before querying $x'$. Therefore, for any realizable auditing problem, there exists an optimal algorithm that adheres to this principle. It is thus encouraging that our greedy algorithm is also consistent with it. 

An $O(\ln(|\cH|_S|))$ approximation factor for auditing is less appealing than the same factor for active learning.  By information-theoretic arguments, active label complexity is at least $\log_2(|\cH|_S|)$ (and hence the approximation at most squares the cost), but this does not hold for auditing. Nonetheless, hardness of approximation results for set cover \citep{feige1998threshold}, in conjunction with the reduction to set cover of \citet{HyafilRi76} mentioned above, imply that such an approximation factor cannot be avoided for a general auditing algorithm.

\section{Conclusion and Future Directions}

In this paper we propose a model of active learning with query costs that depend on the outcome of the query.
We show that the auditing complexity can be significantly different from the active learning complexity, and that tailoring algorithms for auditing can be beneficial.  Our algorithms take advantage of the fact that positive labels are free, to improve error at less cost than in active learning.
We also described a general approach to designing auditing procedures for finite hypothesis classes, based on a greedy strategy and on a partial ordering of points, which takes advantage of the asymmetric label costs.  

There are many interesting directions suggested by this work.  
First, it is known that for some hypothesis classes, active learning cannot improve over passive learning for certain distributions~\citep{Dasgupta04}, and the same is true for auditing.
However, exponential speedups are possible for active learning on certain classes of distributions~\citep{BalcanBeLa06,DasguptaHM:2008}, in particular ones with a small disagreement coefficient~\citep{Hanneke07}. This quantity is related to the Alexander capacity function~\citep{Koltchinskii:10active}, which appears in lower bounds for active learning~\citep{RaginskyR:11nips}. It would be interesting if a similar property of the distribution can guarantee an improvement with auditing over active or passive learning.  

Investigating such a general property might shed light on auditing for other important hypothesis classes such as decision trees or halfspaces. It is well known that for some important settings, such as learning with hyperplanes, there are distributions which resist any improvement using active learning \citep{Dasgupta04}.  Recent work by \citet{GonenSaSh13} has shown that both theoretically and empirically, more aggressive learning strategies can be effective in the realizable case.  These strategies are based on heuristics~\citep{TongK:02active} that query the most ``informative'' points.  It would be interesting to see how such approaches should change in the presence of asymmetric label costs.

The name ``auditing'' suggests an imbalance in the number of points per class (this is the case in fraud).  Prior work on learning from unbalanced data was surveyed by~\citet{HeG:09imbalanced}.  Some of these approaches~\citep{ErtekinHBG:07border} use the same active learning heuristics~\citep{TongK:02active}, and it would be interesting to see how these apply to auditing.

In this work we considered algorithms which aim to minimize the number of negative labels needed to classify all of the points accurately, but this is not the only way to measure the performance in an auditing setting.  For example, we may wish to maximize the number of positive points the algorithm finds subject to a bound on the number of negative labels encountered.  In addition, auditing is an extreme version of asymmetric label costs -- positive labels are free -- but it would be interesting to study more general label costs.  An interesting generalization along these lines is a multiclass setting with a different cost for each label.  These measures and costs are different from those studied in active learning, and may lead to new algorithmic insights.  

%
%
%
%
%
%

%


%
%
%
%

%

\appendix

\section{Proofs omitted from the text}\label{app:proofs}

\subsection{Additional notation}
We use $C, C_1,C_2,\ldots,c,c'$ etc.\ to denote positive numerical constants. Their values may change between expressions. 
We use the shorthand $\forall^\delta$ to say that a statement holds with probability at least $1-c\delta$, for some constant $c$.  
Denote a multiplicative/additive upper bound for $a$ by $a\ap{n,\lambda} = \max\{n a, a + (n-1)\lambda\}$. We will use the following easy fact.

\begin{fact}\label{fact:appr}
For any non-negative numbers $a,b,c,\epsilon,n,m$,
if $a \leq b\ap{n,\epsilon}$ then $a\ap{m,\lambda} \leq b\ap{mn,\lambda}$.
\end{fact}

\subsection{Standard results from probability}

These are included for the ease of the reader.

\begin{theorem}[Hoeffding's Inequality~\citep{Hoeffding63}]
\label{thm:hoeffding}
Let the random variables $X_1, \ldots, X_n$ be independent with $X_i \in [0,1]$, and let $X = \frac{1}{n}\sum_{i \in [n]}X_i$.  Then for any $t > 0$,
	\[
	\mathbb{P}[X > \mathbb{E}[X] + t] \le \exp\left( - 2 nt^2 \right).
	\]
\end{theorem}

\begin{theorem}[Bernstein's Inequality~\citep{Bernstein46}]
\label{thm:bernstein}
Let the random variables $X_1, \ldots, X_n$ be independent with $X_i - \mathbb{E}[X_i] \le 1$.  Let $X = \frac{1}{n}\sum_{i=1}^{n} X_i$ and let $\sigma^2 = \frac{1}{n}\sum_{i=1}^{n} \Var(X_i)$ be the variance of $X$.  Then for any $t > 0$,
	\[
	\mathbb{P}[X > \mathbb{E}[X] + t] \le \exp\left( - \frac{nt^2}{2 (\sigma^2 + t/3) } \right).
	\]
	In particular, by setting the right hand side to $\delta$ and solving for $t$, we get that with probability $1-\delta$,
	\[
	X \leq \E[X] + \frac{2}{3}\ln(1/\delta)/n + \sqrt{2\sigma^2 \ln(1/\delta)/n}.
	\]
\end{theorem}

\subsection{Proofs for Section \ref{sec:thresholds}}\label{app:thresholds}

\begin{proof}[of \lemref{lem:line:subsample}]
We start with the first inequality. If $\emax \geq 1/6$ then the inequality trivially holds.
Thus assume $\emax < 1/6$. 
Let $W = 14\ln(8/\delta)$ and let $N = WT$ be the size of the subset $S_q$. 
Let $h^* \in \argmin_{h\in \cht} \err(S,\cht)$ be a minimizer of the error on $S$. By assumption $\err(S,h^*) \le \emax$.
For each $t$, let $\{ X_t(l) \mid l \in [W]\}$ be the $W$ points drawn from $S(t)$ by the procedure and $Y_t(l)$ be their corresponding labels given by $S$.  Let $Z_t(l) = \one[Y_t(l) \neq h^*(X_t(l))]$  and note that $\{ Z_t(l) \}$ for $l \in [W]$ are i.i.d. random variables, and $Z_t(l) - \E[Z_t(l)] \le 1$.  Furthermore, we have $\Var[Z_t(l)] \leq \E[Z^2_t(l)] \leq \E[Z_t(l)]$.  Therefore 
	\[
	\frac{1}{N} \sum_{t \in [T], l\in[W]} \Var[Z_t(l)] \leq \frac{1}{N} \sum_{t \in [T], l\in[W]} \E[Z_t(l)] \leq \err(S,h^*) \leq \emax.
	\]
Therefore by Bernstein's inequality \citep[see \thmref{thm:bernstein}]{Bernstein46},  with probability $1-\delta$,
	\[
	\err(S_q,h^*) = \frac{1}{N}\sum_{t \in [T],l\in[W]}Z_t(l) 
		\leq \emax + \frac{2}{3}\ln(1/\delta)/N + \sqrt{2\emax \ln(1/\delta)/N}.
\]
Because $T = \max\{\floor{1/3\eta},1\}$, for $\emax < 1/6$ we have $T \geq 1/6\eta$. Therefore
$N \geq 14\ln(8/\delta)/6\eta$. Substituting $N$ and $\delta$ in the inequality above we get that with probability $1-\delta/8$, $\err(S_q, \hat{h}) \le \err(S_q,h^*) \leq 6 \emax$.

For the second claim, if $\emax > 1/17$ the claim trivially holds. Thus assume $\emax \leq 1/17$.
We now show that the error of a threshold $\hat{h} \in \argmin_{h \in \cht} \err(S_q,\cht)$ on the original set $S$ is at most $17\emax$.  Let $A(h) = \{ x \in U \mid h(x) = 1 \}$ be the points in the set $U$ labeled $1$ by a hypothesis $h$, and let $g_i = h_{x'_{(i-1) m+1}}$ be the hypothesis corresponding to the threshold at the leftmost point of $S(i)$.  Note that $A(g_i) = \bigcup_{j > i} S(j)$.

\begin{figure}
\centering
\includegraphics[width=4in]{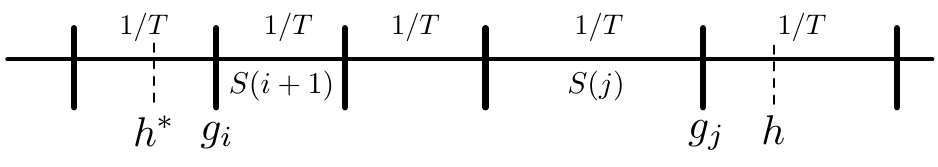}
\caption{Illustration of bound for Lemma \ref{lem:line:subsample}.  The contribution to the error on $U$ in each interval is at most $1/T$.  By assumption, the error for $h$ is $4/T$ more than the error for $h^*$, so  there must be three full intervals between them. \label{fig:thresh}}
\end{figure}

We claim that for any hypothesis $h$ such that $\err(S,h) > \err(S,h^*) + 4/T$, the error on the sampled set $S_q$ must satisfy $\err(S_q,h) > \err(S_q,h^*)$ with high probability, and therefore $h$ cannot be a minimizer $\hat{h}$.  We consider two cases, depending on whether the threshold for $h$ is larger or smaller than $h^*$.  First suppose that the threshold is larger so that $A(h) \subseteq A(h^*)$.  Let $i$ be the smallest index such that $A(g_i) \subseteq A(h^*)$ and $j$ be the largest index such that $A(h) \subseteq A(g_j)$.  The situation is illustrated in Figure \ref{fig:thresh}.  Since $\err(S,h) > \err(S,h^*) + 4/T$, there must be three full intervals $S(t)$ between $g_i$ and $g_j$, so $|j - i| \ge 3$.  Define $\Delta = |j - i|$.

Then using the fact that a single $S(t)$ can contribute at most $1/T$ to the error on $S_q$, we can bound the gap:
	\begin{align*}
	&\err(S_q,h) - \err(S_q,h^*)\\
	&\quad = \err(S_q,h) - \err(S_q,g_j) + \err(S_q,g_j) - \err(S_q,g_i) + \err(S_q,g_i)- \err(S_q,h^*) \\
	&\quad \geq \err(S_q,g_j) - \err(S_q,g_i) - 2/T. 
	\end{align*}
Therefore for any $h$ whose threshold is between those for $g_j$ and $g_{j+1}$, in order to show that $\err(S_q,h) > \err(S_q,h^*)$ with high probability it is sufficient to show that $\err(S_q,g_j) - \err(S_q,g_i) \ge 2/T$ with high probability.

Note that the number of points in $S_q$ on which $g_i$ and $g_j$ disagree is $W \Delta$, corresponding to the subsamples in the $\Delta$ segments $S(i+1), S(i+2), \ldots, S(j)$ in Algorithm \ref{alg:subset}.  For each pair $(x,y)$ in those segments, either $g_i$ or $g_j$ errs, and $\err((x,y),g_j) - \err((x,y),g_i) = 1 - 2 \one[ g_i(x) \ne y ]$.  Let $Z_t^i(l) = \one[Y_t(l) \neq h_i(X_t(l))]$. Then
	\begin{align*}
	\err(S_q,g_j) - \err(S_q,g_i) &= \frac{1}{WT} \sum_{t=i+1}^{j} \sum_{l \in [W]} ( 1 - 2 Z_t^i(l) )= \frac{ \Delta }{ T }  - \frac{2}{WT} \sum_{t=i+1}^{j} \sum_{l \in [W]} Z^i_t(l).
	\end{align*}
The event that this difference is smaller than $2/T$ is equivalent to
	\begin{align*}
	\frac{1}{W \Delta } \sum_{t=i+1}^{j} \sum_{l \in [W]} Z^i_t(l) \ge \frac{ \Delta  - 2 }{ 2 \Delta } \ge \frac{1}{6}.
	\end{align*}	
Note that $h_i$ agrees with $h^*$ on $\bigcup_{t=i+1}^{j} S(t)$, so
	\[
	\E\left[ \frac{1}{W \Delta } \sum_{t=i+1}^{j} \sum_{l \in [W]} Z^i_t(l) \right] \le \err(S,h^*) \leq \emax.
	\]
By Hoeffding's inequality \citep[see \thmref{thm:hoeffding}]{Hoeffding63}, and since $\emax \le 1/17$,
	\begin{align*}
	\P[\err(S_q,g_j) - \err(S_q,g_i) \le 2/T]
	&\le 
	\P\left[\frac{1}{W \Delta } \sum_{t=i+1}^{j} \sum_{l \in [W]} Z'_t(l) \ge \frac{1}{6} \right] \\
	&\le \exp\left( - 2 W \Delta \left( \frac{1}{6} - \emax \right)^2 \right) \\
	&\le \exp\left( - W \Delta / 42 \right).
	\end{align*}
Now taking a union bound over all $j$ such that $j > i + 3$, we have
	\begin{align*}
	\P[\forall j \ge i + 3,\err(S_q,g_j) - \err(S_q,g_i) \le 2/T ]
	&\le \sum_{\Delta = 3}^{T} \exp\left( - W \Delta / 42 \right) \\
	&\le \frac{ \exp( - W/14 ) - \exp( - W(T+1)/42 ) }{ 1 - \exp( - W/42 ) } \\
	&\le \frac{ \exp( - W/14 ) }{ 1 - \exp( - W/42) }.
	\end{align*}
The other case when $h < h^*$ is symmetric, so we see that if $\err(S,h) > \err(S,h^*) + 4/T$ then
	\begin{align*}
	\P[ \err(S_q,h) > \err(S_q,h^*) ]
	&\le 2 \frac{ \exp( - W/14 ) }{ 1 - \exp( - W/42 ) }.
	\end{align*}
Since $W = 14\ln(8/\delta)$, we get that the total probability is bounded by $\delta/2$. Since $T > \frac{1}{3 \emax} - 1$, we have for $\emax \le 1/17$ that
$T > \frac{1}{3\emax} - \frac{1}{17\emax} \geq 1/4\emax$. Therefore for $\hat{h}$ which minimizes the error on $S_q$ we have $\err(S,\hat{h}) < \err(S,h^*) + 4/T < 17 \emax$.
\end{proof}

To prove \thmref{thm:threshalpha}, we require the following lemma.

\begin{lemma}\label{lem:algthresh}
For $S_0$ and $h_{\hat{\va}}$ in \algref{alg:thresh}, if $\err(S_0,\cht) \le (1+\nu)\emax$, then the auditing complexity of step \ref{step:sq} of \algref{alg:thresh} is at most $85\ln(16/\delta)$ and with probability $1-\delta$, $\err(S_0,h_{\hat{\va}}) \leq 35(1+\nu)\emax$.
\end{lemma}

\begin{proof}
Denote $\gamma = (1+\nu)\gamma$.
Let $h^* \in \argmin_{h \in \cH}\err(S_0,\cH)$. 
Since $|S| = m^{\mathrm{ag}}(\gamma,\delta/2,1)$, with probability $1-\delta/2$, $\err(S,\cht) \leq \err(S_0,\cH) + \gamma \leq 2\gamma$.
By \lemref{lem:line:subsample}, the total number of negative errors in $S_q$ is at most $\ceil{12\gamma\cdot |S_q|}+1$. Therefore, by \lemref{lem:line:pool}, step \ref{step:sq} finds a hypothesis $h_{\hat{\va}}$, that minimizes the error on $S_q$. 
By \lemref{lem:line:subsample}, with probability $1-\delta/2$, $\err(S,h_{\hat{\va}}) \leq 34\gamma$. Therefore, due to the size of $|S|$ again, with probability $1-\delta$, $\err(S_0,h_{\hat{\va}}) \leq 35\gamma$.

The auditing complexity of step \ref{step:sq} is at most $6\gamma\cdot |S_q|+1$. Now, from \algref{alg:subset},
$|S_q| \leq 14\ln(16/\delta)\cdot \max\{\floor{1/3\gamma},1\}$. 
Since $\gamma\cdot \max\{\floor{1/3\gamma},1\} \leq 1$, the auditing complexity of \algref{alg:thresh} is at most 
$\ceil{6\gamma\cdot |S_q|}+1 \leq 85\ln(16/\delta)$.
\end{proof}

We are now ready to prove the theorem. 
\begin{proof}[of \thmref{thm:threshalpha}]
We first bound $\err(D,\hat{h})$. 
Let $h^* \in \argmin_{h \in \cH}\err(D,h)$, and $h^*_0 \in \argmin_{h \in \cH}\err({S_0},h)$. 
Since $|{S_0}| = m_{\nu}(\emax,\delta/2,1)$, with probability $1-\delta/2$,
\begin{equation}\label{eq:deq}
\err({S_0},h^*_0)  \leq \err({S_0},h^*) \leq \err(D,\emax)\ap{(1+\nu),\emax} \leq (1+\nu)\emax.
\end{equation}
Therefore, by \lemref{lem:algthresh}, $h_{\hat{\va}}$ satisfies $\forall^\delta, \err({S_0},\hat{h}) \leq 35(1+\nu)\emax$. 
It follows that 
\[
\P_{(X,Y)\sim {S_0}}[h^*_0(X) \neq \hat{h}(X)] \leq \err({S_0},\hat{h}) + \err({S_0},h^*_0) \leq 36(1+\nu)\emax.
\]
In other words, $h^*_0$ classifies at most $36(1+\nu)\emax|{S_0}|$ points differently from $h_{\hat{\va}}$. Therefore $h^*_0 \in \argmin_{h\in \cH}\err(S_1,h)$, where $S_1$ is defined in step \ref{step:s1}, since all points in ${S_0}\setminus S_1$ are classified the same by all possible candidates for $h^*_0$. 

We have 
\begin{equation}\label{eq:ers1}
\err(S_1,h^*_0) \leq \frac{|{S_0}|}{|S_1|}\err({S_0},h^*_0) \leq \frac{|{S_0}|}{2\cdot36(1+\nu)\emax|{S_0}|}(1+\nu)\emax \leq \frac{1}{72}.
\end{equation}
Since $|S_2| = m^{\mathrm{ag}}(\nu/72,\delta/2,1)$, with probability $1-\delta/2$, 
\[
\err(S_1,\hat{h}) \leq \err(S_2,\hat{h})+\nu/72 \leq \err(S_2,h^*_0)+ \nu/72,
\]
 and $\err(S_2,h^*_0) \leq \err(S_1,h^*_0)\ap{(1+\nu),\frac{1}{72}}$. Therefore 
\[
 \forall^\delta, \err(S_1,\hat{h}) \leq \err(S_1,h^*_0)\ap{(1+\nu),\frac{1}{72}} +\nu/72 \leq 
 \err(S_1,h^*_0) + \nu/36,
\]
where the last inequality follows from \eqref{eq:ers1}.
Note also that $\err({S_0}\setminus S_1,\hat{h}) = \err({S_0}\setminus S_1,h^*_0)$.
\begin{align*}
\err({S_0},\hat{h}) &= \frac{|{S_0}|-|S_1|}{|{S_0}|}\err({S_0}\setminus S_1,h^*_0) + \frac{|S_1|}{|{S_0}|}\err(S_1,\hat{h})\\
&\leq \frac{|{S_0}|-|S_1|}{|{S_0}|}\err({S_0}\setminus S_1,h^*_0) + \frac{|S_1|}{|{S_0}|}(\err(S_1,h^*_0) + \nu/36)\\
&= \err({S_0},h^*_0) + 72(1+\nu)\emax(\nu/36)\\
&\leq \err({S_0},h^*_0) + 4\nu\emax.
\end{align*}
In the last inequality we used the fact that $\nu \leq 1$. Therefore
$\err({S_0},\hat{h}) \leq \err({S_0},h^*_0) + 4\nu\emax.$
Combining this with \eqref{eq:deq} we conclude that with probability $1-\delta$, $\err({S_0},\hat{h}) \leq \emax(1+5\nu)$. Since $\nu = \alpha/5$,
this proves the lemma.

The auditing complexity of \algref{alg:thresh} is at most the auditing complexity of step \ref{step:sq}, which is $O(\ln(1/\delta))$ by \lemref{lem:algthresh}, plus $m^{\mathrm{ag}}(\nu/72,\delta/2,1) = O(\ln(1/\delta)/\nu^2) O(\ln(1/\delta)/\alpha^2)$. Thus the total auditing complexity is also $O(\ln(1/\delta)/\alpha^2)$.
\end{proof}

\begin{proof}[of \thmref{thm:lowerthresh}]
Fix $\emin$ and define $\beta = \alpha + 1$.  Assume without loss of generality that the algorithm returns a hypothesis $\hat{h}$ after having queried exactly $M$ negative labels.  We will define a finite set of distributions such that if the algorithm emits a correct answer for all of them, then the given lower bound on $M$ must hold.

We consider distributions with a uniform marginal over $[0,1]$, and define several conditional labeling distributions for points on $[0,1]$.  
First, we define the distribution $D_-$, which assigns $-1$ to all $x \in [0,1-2\emin] \cup [1-\emin,1]$, and $+1$ to $x \in (1-2\emin,1-\emin)$.  Note that $\err(D_-,\cht) = \emin$, so the guarantee of the algorithm is that $\err(\hat{h}, D_-) \leq \beta \emin$ with probability $1-\delta$. Thus for $D_-$ the algorithm produces a hypothesis $\hat{h} = h_a$ for some threshold value $a \ge 1 - (1 + \beta) \emin$ with probability $1 - \delta$.

Second, we define a family of distributions $D_1,\ldots,D_N$, for $N = \floor{\ln(1/2\eta\beta)/\ln(4\beta)}$, such that for each $D_i$, the algorithm cannot return $h_a$ with $a \geq 1-(\beta+1)\eta$ with probability more than $\delta$.  

\begin{figure}[bh]
\begin{center}
\begin{tikzpicture}[scale = 1.5]
\draw[-] (0,0) -- (5,0);
\draw[-] (0,-0.1) -- (0,+0.1);
\draw[-] (5,-0.1) -- (5,+0.1);
\draw[-] (4.5,-0.1) -- (4.5,+0.1);
\draw[-] (4,-0.1) -- (4,+0.1);
\draw[-] (3,-0.1) -- (3,+0.1);
\draw[-] (3,-0.1) -- (3,+0.1);
\node at (4.75,-0.2) {$\eta$};
\node at (4.25,-0.2) {$\eta$};
\node at (3.55,-0.2) {$(\beta\!\!-\!\!1)\eta$};
\node at (3,-0.2) {$a_0$};
\node at (4.75,0.2) {0};
\node at (4.25,0.2) {1};
\node at (3.5,0.2) {0};
\draw[-] (1,-0.1) -- (1,+0.1);
\draw[-] (2.5,-0.1) -- (2.5,+0.1);
\node at (1.5,0.2) {$1-\alpha$};
\node at (0.5,0.2) {0};
\node at (2.75,0.2) {0};
\node at (1,-0.2) {$a_i$};
\node at (2.5,-0.2) {$a_{i-1}$};
\end{tikzpicture}
\end{center}
\caption{The probability of a positive label for $D_i$ (not in scale) \label{fig:thresh:lower}}
\end{figure}
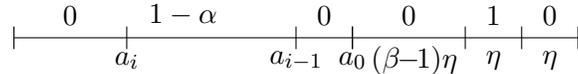

Let $\lambda = 1/8\beta$.  Define $a_0 = 1-(1 + \beta)\emin$, and for $i \in [N]$ define $l_i = \beta (4 \beta)^{i} \emin$ and $a_i = a_0 - \sum_{j \le i} l_j$.  Define the distribution $D_i$ as follows (See Figure \ref{fig:thresh:lower}): 
\[
\P_{D_i}[Y = +1 \mid X = x] = \begin{cases} 
0 & x \in [0,a_i]\cup[a_{i-1}, 1-2\emin]\cup [1-\emin,1]\\
1 & x \in (1-2\emin,1-\emin)\\
1-\lambda & x \in (a_i,a_{i-1}).
\end{cases}
\]
The distribution $D_i$ agrees with the distribution $D_-$ except on the interval $(a_i,a_{i-1})$, where it is positive with probability $1 - \lambda$ and negative with probability $\lambda$.  We claim that if the algorithm returns a threshold greater than $a_0$ on $D_i$ with probability more than $\delta$, it violates the $(\alpha,\delta)$-learning guarantee.  For $a_0$, and $\beta \ge 1$,
	\begin{align*}
	\err(D_i, h_{a_0}) &\ge (1 - \lambda) l_i > \frac{7}{8} \beta (4 \beta)^i \emin.
	\end{align*}
For $a_i$,
	\begin{align*}
	\err(D_i, h_{a_i}) &= \beta \emin + \sum_{j < i} l_j + \lambda l_i  \\
		&= \left(\sum_{j=0}^{i-1} \beta (4 \beta)^{j} + \frac{1}{8 \beta} \beta (4 \beta)^{i} \right) \emin \\
		\end{align*}
		Hence 
		\begin{align*}
		\err(D_i, h_{a_i}) &= \beta \left(\frac{ (4 \beta)^i - 1 }{ 4 \beta - 1 } + \frac{1}{8} (4 \beta)^{i} \right) \emin \\
		&< \frac{7}{8} \beta (4 \beta)^i \left( \frac{8}{7 (4 \beta - 1)} +  \frac{1}{7} \right) \emin \\
		&< \frac{7}{8} (4 \beta)^i.
	\end{align*}
From this we can see that $\err(D_i, h_{a_0}) > \beta \err(D_i, h_{a_i})$,
violating the guarantee of the algorithm.  It follows that for any $i$, if the true labeling is $D_i$, then the probability that the algorithm returns $h_a$ for $a \geq 1-(\beta+1)\eta$ is at most $\delta$. We now show that this implies a lower bound on $M$.

First, since all distributions label $[a_0,1]$ in the same way, we may assume without loss of generality that the algorithm never queries points in this segment. It follows that if the true distribution is $D_-$, the algorithm observes only negative labels.

Denote by $Y_t$ the random variable whose value is the label the algorithm receives for its $t$'th query, or $0$ if the algorithm stopped before querying $t$ points. Denote by $Z_t$ the random variable whose value is $j$ if on round $t$, the algorithm queries a point in $[a_j,a_{j-1}]$, and $-1$ if the algorithm stops before round $t$. 
Denote by $A_t$ the event that $\forall i \in [t], Y_i = -1$. Also denote $p_j^t = \P[Z_t = j \mid A_{t-1}]$.
We will show a lower bound on $\P[A_M]$, that is the probability that all first $M$ queries return a negative label.
Since in this case the algorithm cannot distinguish $D_j$ from $D_-$, this probability must be small, which implies a lower bound on $M$.

By definition, $\sum_{j\in [N]}p_j^t = 1$ for all $t \leq M$. Thus, there exists some $j \in [N]$ such that 
$\sum_{t \in [M]}p_j^t \leq M/N$. Fix $j$ to one such value.
Assume that the true labeling is $D_j$. Then 
	\begin{align*}
	\P[A_1] &= \lambda p_j^1 + (1-p^1_j) = 1 - (1-\lambda)p_j^1, \\
	\P[A_t] &= \P[A_{t-1}]\P[Y_t = -1 \mid A_{t-1}] = \P[A_{t-1}](1 - (1-\lambda)p_j^t).
	\end{align*}
It follows that if the true distribution is $D_j$, then 
\[
\P[A_M] = \prod_{t \in [M]}(1 - (1-\lambda)p_j^t).
\] 
We consider two kinds of indices $t \in [M]$.  First let \mbox{$I_j = \{t\in [M]\mid p_j^t > 1/2(1-\lambda)\}$}.  Since $\sum_{t \in [M]} p_j^t \leq M/N$,
we have \mbox{$|I_j| \leq 2(1-\lambda)M/N$}.  For $t \in I$ we use the bound $1 - (1-\lambda)p_j^t \geq \lambda$.  For $t \notin I$, we use the bound $1 - (1-\lambda)p_j^t \geq \exp(-2(1-\lambda)p_j^t).$
This follows from the inequality $\exp(-2x) \leq 1-x$, which holds for $x \in [0,\half]$. 
Combining the two cases, we get
\begin{align}
\P[A_M] &= \prod_{i \in [M]}(1 - (1-\lambda)p_j^t) \geq \exp\left(-2(1-\lambda)\sum_{t \notin I}p_j^t \right)\lambda^{|I|} \nonumber \\
&\geq \exp\left(-2(1-\lambda)\frac{M}{N} \right)\lambda^{2(1-\lambda)(M/N) } \nonumber \\
&= \exp\left(-2(1-\lambda)\frac{M}{N}(1+\ln(1/\lambda) )\right). \label{eq:Mnegbound}
\end{align}

The algorithm must stop after seeing $M$ negative labels, thus it must return an answer at time $M$ if $A_M$ occurs. If the true distribution is $D_-$, then $A_M$ occurs with probability $1$. Therefore, if $A_M$ occurs the algorithm must return $h_a$ for $a \geq a_0$ with probability at least $1-\delta$. 
Thus, if the true distribution is $D_j$, the probability that the algorithm errs is at least $\P[A_M](1-\delta)$.
Since the algorithm errs with probability at most $\delta$, we have
\[
\delta \geq \P[A_M](1-\delta), %
\]
Solving for $M$ using \eqref{eq:Mnegbound}, we get 
\[
M \geq \frac{N \ln(\frac{1-\delta}{\delta})}{2(1-\lambda)(1+\ln(1/\lambda))}.
\]
Treating $\beta$, and hence $\lambda$, as constants, we get that $N \geq C\ln(1/\emin)-C'$, therefore
$M \geq C\ln(\frac{1-\delta}{\delta})\ln(1/\emin)-C'$ for some positive constants $C,C'$.
\end{proof}

\subsection{Proofs for Section \ref{sec:aprs}}

Here we gather proof details for the hypothesis class $\chr$ and $\chr^{-}$ of axis-aligned rectangles.

\subsubsection{Proof of Theorem \ref{thm:apr:ag:lower} \label{thm:apr:ag:lower:pf} }

\begin{proof}[of Theorem \ref{thm:apr:ag:lower}]
We will show that in the realizable case, an algorithm that returns $\hat{h}$ such that $\forall^\delta \err(D,\hat{h}) \leq \epsilon$ requires $\Omega(d/\epsilon)$ labels. The statement of the theorem follows by adding an unavoidable error of $\eta$ to all distributions.

Without loss of generality, suppose $d$ is even and $1/4\epsilon$ is an integer, and partition the $d$ dimensions in $d/2$ pairs of coordinates $\{(1,2), (3,4), \ldots, (d-1,d)\}$.  For each coordinate pair $(2 i - 1,2 i)$ choose $\frac{1}{4\epsilon}$ distinct points $S_i$ on the unit circle in the subspace spanned by the $i$ and $(i+1)$-th coordinates,
as in the proof of \thmref{thm:aprlower}.
Consider distributions $D$ with a uniform marginal over the points in $S_1,\ldots,S_{d/2}$, so that the probability of each point is $4\epsilon/d$. Any distribution such that all points are labeled negative, except perhaps a single point in every $S_i$, is realizable. To get $\err(D,\hat{h}) \leq \epsilon$ with probability more than half, the algorithm must find whether there is a positive point in at least half of the $S_i$'s.

Let $T_i$ be the number of points queried by the algorithm in set $S_i$.  If the total number of queries that the algorithm makes is less than $d|S_i|/16$, then $\E[T_i] < |S_i|/8$ for at least half of the $i$'s.
If $\E[T_i] < |S_i|/8$ then with probability at least $1/2$, $T_i \leq 1/4$. Thus there exists a point in $S_i$ such that with probability at least $1/2$ the algorithm does not query this point, and therefore cannot tell whether it is positive. It follows that the algorithm must make at least $d|S_i|/16 = \Omega(d/\epsilon)$ queries on negative points.
\end{proof}

\subsubsection{Approximation bounds for error on samples \label{lem:errneg:pf} }

\begin{lemma}\label{lem:errneg}
Let $\cH$ be a hypothesis class with VC dimension $d \geq 1$, and let $S$ be a sample of size $m_\nu(\epsilon,\delta,d)$ drawn i.i.d.\ from a distribution $D$. With probability $1-\delta$, $\forall h \in \cH, $
\begin{align*}
&\errneg(S,h)\leq \errneg(D,h)\ap{(1+\nu),\epsilon} \text{ and }\errneg(D,h)\leq \errneg(S,h)\ap{(1+\nu),\epsilon}.
\end{align*}
\end{lemma}

\begin{proof}
Let $f[h]:(\reals_+^d \times \binlab) \rightarrow \binlab$ such that 
$f[h](x,y) = \one[h(x) = 1 \text{ and } y = -1]$. For any distribution over $\reals_+^d \times \binlab$, consider a distribution $D'$ over $(\reals_+^d \times \binlab) \rightarrow \binlab$ that draws $((X,Y),Z) \sim D'$ such that $(X,Y)$ is drawn from $D$ and $Z$ is deterministically $1$. 
Then $\errneg(D,h) = \err(D',f[h])$. The VC-dimension of $\cF = \{ f[h] \mid h \in \cH\}$ is at most that of $\cH$:
Any set $((x_1,y_1),\ldots, (x_n,y_n))$ shattered by $\cF$ must have $y_i = -1$ for all $i \in [n]$. Therefore
$\forall h \in \cH, f[h](x_i,y_i) = h(x_i)$, hence $x_1,\ldots,x_d$ is shattered by $\cH$. The result follows by applying \eqref{eq:vc} to $\err(D',f[h])$.
\end{proof}

\subsubsection{Proof of Theorem \ref{thm:aprmain} \label{thm:aprmain:pf}}

\thmref{thm:aprmain} is proven using several lemmas.   We will need the following auxiliary result.

\begin{lemma}\label{lem:vcf}
Let $\cH$ be a hypothesis class of VC-dimension $d$, and let $f[h_1,h_2]:(\reals_+^d \times \binlab) \rightarrow \binlab$ be the function $f[h_1,h_2](x) = \one[h_1(x) = y \text{ or } h_2(x) = 1]$. 
The VC-dimension of $\cF = \{ f[h] \mid h \in \cH\}$ is at most $10d$.
\end{lemma}

\begin{proof}
Let $S = ((x_1,y_1),\ldots, (x_n,y_n))$ be a set shattered by $\cF$. Then $|\cF|_S| = 2^n$. In addition, 
$|\cF|_S| \leq |\cH|_S \times \cH|_S| \leq |\cH|_S|^2$. By Sauer's lemma, $|\cH|_S| \leq (en/d)^d$.
Therefore $2^n \leq (en/d)^{2d}$. It follows that $n \leq 10d$.
\end{proof}

The next lemma will help prove that the set of hypotheses maintained by the algorithm includes the best hypothesis for the distribution.

\begin{lemma}\label{lem:hstarinv}
Let $\nu,\epsilon > 0$ and $\delta \in (0,1)$. Let $S$ be a random labeled sample of size $m_\nu(\epsilon,\delta,10d)$ drawn from $D$ . 
For $\vb \in \reals_+^d$, let $S_\vb$ be identical to sample $S$  
except that if $(x,y) \in S$ and $x \leq \vb$, then $(x,-1) \in S_\vb$ instead of $(x,y)$.
Let $h^* \in \argmin_{h\in\chr}  \err(D,h)$, and let $\va^*$ such that $h^* = h_{\va^*}$.
Let $\hat{h}_\vb = \argmin_{h \in \chr[\vb]}\err(S_\vb, h)$.
Then $\forall^\delta$, for all $\vb \leq \va^*$,
\[
h^* \in V_\nu(S_\vb,\epsilon,\chr[\vb]), \text{ and } \err(D,\hat{h}_\vb) \leq \err(D,h^*)\ap{(1+\nu)^2,\epsilon}.
\]
\end{lemma}

\begin{proof}[of Lemma \ref{lem:hstarinv}]
For the first claim, it suffices to show that $\forall^\delta$, for all $\vb \leq \va^*$,
	\begin{equation}\label{eq:v4}
	\err(S_\vb,h^*) \leq \err(S_\vb,\hat{h}_\vb)\ap{(1+\nu)^2, \epsilon}.
	\end{equation}
Define $f_{\va,\vb}:\reals_+^d \times \binlab \rightarrow \binlab$ such that 
$f_{\va,\vb}(x,y) = \one[h_\va(x) = y \text{ or } h_\vb(x) = -1]$.
Let $S' = \{ ((x,y),1) \mid (x,y) \in S\}$, and let $D'$ be a distribution over $(\reals_+^d \times \binlab) \times \binlab$ generated by drawing $((X,Y),Z) \sim D'$ where $(X,Y) \sim D$ and $Z = 1$.
Then $S'$ is drawn i.i.d. from $D'$. 
Note that for any $\va \geq \vb$, $h_\va$ classifies all points $x\leq \vb$ as negative. It follows that there is some $\lambda > 0$ such that for all $\va\geq \vb$, $\lambda = \err(D,h_{\va}) - \err(D',f_{\va,\vb})$. 

The VC-dimension of $\cF = \{ f_{\va,\vb} \mid \va \geq \vb\}$ is at most $10d$ (see \lemref{lem:vcf} in the appendix).
Since $|S'| \geq m_\nu(\epsilon,\delta,10d)$, 
$\forall^\delta,\forall f \in \cF, \err(S', f) \leq \err(D',f)\ap{1+\nu,\epsilon}$  and $\err(D', f) \leq \err(S',f)\ap{1+\nu,\epsilon}.$
Let $\hat{\va}_\vb \in \reals_+^d$ such that $\hat{h}_\vb = h_{\hat{\va}_\vb}$. 
We have $\err(D,h_{\va^*}) \leq \err(D,h_{\hat{\va}_\vb})$, therefore $\err(D',f_{\va^*,\vb}) \leq \err(D',h_{\hat{\va}_\vb,\vb})$. Combining these inequalities and using \factref{fact:appr}, we get 
	\begin{align*}
	\forall^\delta,\forall \vb\in \reals_+^d, \qquad \err(S', f_{\va^*,\vb}) &\leq \err(D',f_{\va^*,\vb})\ap{1+\nu,\epsilon} \\
	&\leq \err(D',f_{\hat{\va}_\vb,\vb})\ap{1+\nu,\epsilon} \\
	&\leq \err(S',f_{\hat{\va}_\vb,\vb})\ap{(1+\nu)^2,\epsilon}.
	\end{align*}
Noting that for $\va \geq \vb$, $\err(S',f_{\va,\vb}) = \err(S_\vb,h_{\va})$, this proves \eqref{eq:v4}.

For the second claim, 
	\begin{align*}
	\forall^\delta,\forall \vb\in \reals_+^d, \qquad
	\err(D', f_{\hat{\va}_\vb,\vb}) &\leq \err(S',f_{\hat{\va}_\vb,\vb})\ap{1+\nu,\epsilon} \\
	&\leq \err(S',f_{\va^*,\vb})\ap{1+\nu,\epsilon} \\
	&\leq \err(D',f_{\va^*,\vb})\ap{(1+\nu)^2,\epsilon}.
\end{align*}
Denoting $\lambda = \err(D,h_{\va^*}) - \err(D',f_{\va^*,\vb}) = \err(D,h_{\hat{\va}_\vb}) - \err(D',f_{\hat{\va}_\vb,\vb})$,
we get 
	\[
	\err(D,h_{\hat{\va}_\vb}) - \lambda \leq (\err(D,h_{\va^*})-\lambda)\ap{(1+\nu)^2,\epsilon}.
	\]
Since $\lambda > 0$, this implies $\err(D,h_{\hat{\va}_\vb}) \leq \err(D,h_{\va^*})\ap{(1+\nu)^2,\epsilon}.$
\end{proof}

The following lemma shows that $\eta_t$ is indeed an upper bound for the negative error of the best hypothesis. 

\begin{lemma}\label{lem:claims}
If the algorithm reaches round $T$, then $\forall^\delta, \forall t \leq T, \forall h^* \in \argmin_{h\in \chr}\err(D,h)$, the following claims hold:
\begin{itemize}
\item Claim $A(t)$: $\errneg(D,h^*) \leq \eta_t$.
\item Claim $B(t)$: $h^* \in V_t$, where $V_t$ is defined in step \ref{step:vt} of Algorithm \ref{alg:aprs}
\item Claim $C(t)$: $\errneg(D, h^*) \leq \hat{\eta}_t\ap{1+\nu,\eta_t}$.
\end{itemize}
\end{lemma}

\begin{proof}[of Lemma \ref{lem:claims}]
We will prove the claims by induction on $t$. At each round $t \leq T \leq 1/\log_2(1/\emin)$, $|S_t| \geq m(\eta_t,\delta/\log_2(1/\emin),d)$, thus $\forall^\delta$, uniform convergence as stated in \eqref{eq:vc} holds for all rounds simultaneously. We assume this for the rest of the proof .

First, claim $A(0)$ trivially holds since $\eta_0 = 1$.
It is also easy to see that if claim $C(t)$ holds, and the algorithm reaches round $t+1$,
then claim $A(t+1)$ holds: If \algref{alg:aprs} reached round $t+1$, 
then the condition in step \ref{step:if} failed at time $t$, meaning $\hat{\eta}_t \leq \eta_t/4$. By $C(t)$, 
 $\errneg(D,h^*) \leq \hat{\eta}_t\ap{1+\nu,\eta_t} \leq \eta_t/2 = \eta_{t+1}$ (since $\nu \leq 1$), which proves $A(t+1)$.

We have left to show that claim $A(t)$ implies claims $C(t)$ and $B(t)$. Assume that $A(t)$ holds. 
First, suppose not all the points in $S_t$ are queried. 
To prove $B(t)$, note that $\errneg(S_t,h^*) \leq \errneg(D, h^*)\ap{1+\nu,\eta_t}$. Since $\errneg(D,h^*) \leq \eta_t$, this implies an upper bound $\errneg(S_t,h^*) \leq (1+\nu)\eta_t$.

We now show that $h^* \in \chr[\vb_t]$. 
If all the points in $S_t$ are queried, then $\vb_t$ is the zero vector, thus $\chr[\vb_t] = \chr$ and $h^* \in \chr[\vb_t]$.  If not all the points in $S_t$ are queried, then the algorithm queried more than $(1+\nu)\eta_t|S_t|$ negative points in each direction, thus at least one of those points is labeled negative by $h^*$. The smallest value of coordinate $i$ queried in iteration $i$ of round $t$ is $b_t[i]$.  Therefore the threshold of $h^*$ in direction $i$ is at most $b_t[i]$. This implies $h^* \in \chr[\vb_t]$ and furthermore that $h^* = h_{\va^*}$ for some $\va^* \geq \vb_t$.
By \lemref{lem:hstarinv}, $h^* \in V_\nu(S_{\vb_t}, 4,\epsilon,\chr[\vb_t])= V_t$.
This proves $B(t)$.

For $C(t)$, note that $\hat{\eta}_t = \max_{h \in V_t} \errneg(S_{\vb_t},h)$, hence by $B(t)$, $\hat{\eta}_t \geq \errneg(S_{\vb_t}, h^*) = \errneg(S_t,h^*)$. The claim follows since $\errneg(D,h^*)\leq \errneg(S_t,h^*)\ap{1+\nu,\eta_t}$.
\end{proof}
The last lemma provides a the stopping condition of the algorithm.
\begin{lemma}\label{lem:stop}
If $\err(D,\cH) > \emin$ then $\forall^\delta$ the algorithm stops at round at least $\log_2(1/8\err(D,\cH))$.
\end{lemma}
\begin{proof}
Let $T = \log_2(1/8\err(D,\cH))$. We show that $\forall^\delta$ the algorithm does not stop at any $t \leq T$, by showing that the condition in step \ref{step:if} does not hold, that is $\hat{\eta}_t = \max_{h \in V_t} \errneg(S_{\vb_t},h) \leq \eta_t/4$. By \lemref{lem:claims}, claim $B(t)$, $h^* \in V_t$. Therefore,
by definition of $V_t$, for all $h \in V_t$, 
\[
\errneg(S_{\vb_t},h) \leq \err(S_{\vb_t},h) \leq \err(S_{\vb_t},h^*)\ap{(1+\nu)^2,\eta_t} \leq \err(S_t,h^*)\ap{(1+\nu)^2,\eta_t}.
\]
Due to the size of $S_t$ we also have $\forall^\delta, \forall t\leq T, \err(S_t,h^*) \leq \err(D,h^*)\ap{1+\nu,\eta_t}$. Combining these inequalities we get 
$\hat{\eta}_t \leq \err(D,h^*)\ap{(1+\nu)^3,\eta_t}$. For $t \leq T$, $\err(D,h^*) \leq 2^{-t}/8 = \eta_t/8$.
$\hat{\eta}_t \leq \err(D,h^*) + ((1+\nu^3)-1)\eta_t \leq \eta_t(1/8 + ((1+\nu^3)-1)).$ 
Since $\nu \leq 1/25$, one can check that $\hat{\eta}_t \leq \eta_t/4$.
\end{proof}

We are finally ready to prove \thmref{thm:aprmain}.

\begin{proof}[of \thmref{thm:aprmain}]
Let $T$ be the round in which the algorithm returns $\hat{h}$.
The number of negative labels $N$ observed by the algorithm satisfies
\begin{align*}
N \leq &\sum_{t=0}^T  d(1+\nu)\eta_t(\ceil{m_\nu(\eta_t,\delta/\log_2(1/\emin),10d)}+1)
\end{align*}
By the definition on page \pageref{page:vc}, $m_\nu(\eta,\delta,d) = C(d\ln(c/\nu\eta)+\ln(c/\delta))/\nu^2\eta$. Also $1+\nu \leq 2$. Therefore
\begin{align*}
&N\leq C(T + \frac{d}{\nu^2} \sum_{t=0}^T(d\ln(c/\nu\eta_t) + \ln(c\log_2(1/\emin)/\delta))\\
&\leq Cd(d\sum_{t=0}^T\ln(c/\nu\eta_t)+T\ln(c\ln(1/\emin)/\delta)).
\end{align*}
We have $\sum_{t=0}^T\ln(c/\eta_t) \leq C\sum_{t=0}^T t + C \leq CT^2+C$.
In addition, $T \leq \log_2(1/\emin)$. Therefore
the algorithm observes at most $C d^2\ln^2(1/\emin)\ln(c/\nu\delta)/\nu^2$ negative examples.
Since $\nu = \alpha/25$, we get the same auditing complexity for $\alpha$.

For the second part of the theorem, by \lemref{lem:claims}, $\forall^\delta \argmin_{h\in\chr}\err(D,h) \in V_T = V(S_{\vb_T}, \eta_T,\chr[\vb_T])$. Therefore, by \lemref{lem:hstarinv}, $\err(D,\hat{h}) \leq \err(D,\chr)\ap{(1+\nu)^2,\eta_T}.$
By \lemref{lem:stop}, we have that $T \geq \min\{\log_2(1/8\err(D,\cH),\log_2(1/2\emin)\}$, and therefore $\eta_T \leq \max\{8\err(D,\chr),2\emin\}$. 
It follows that 
	\begin{align*}
	\forall^\delta \quad \err(D,\hat{h}) &\leq \max\big\{(1+\nu)^2 \err(D,\chr), \\
	& \qquad\qquad\err(D,\chr)+ ((1+\nu)^2-1) \cdot \max\{8\err(D,\chr),2\emin\} \big\} \\
	&\leq \max\{(1+8((1+\nu)^2-1))\err(D,\chr),\err(D,\chr)+ 2((1+\nu)^2-1) \emin \big\} \\
	& \leq \err(D,\chr)\ap{(1+8((1+\nu)^2-1)),\emin}.
	\end{align*}
	Now, since $\nu \leq 1$ and $\nu = \alpha/25$, we have $8((1+\nu)^2-1) = 8(2\nu + \nu^2) \leq 24\nu \leq \alpha$. The statement of the theorem immediately follows.
\end{proof}

\subsection{Proofs for Section \ref{sec:general}}

%
%

\begin{proof}[of \thmref{thm:greedy}]
Assume without loss of generality that $\cH|_S = \cH$.
For an algorithm $\cA$, let $Q^k_{\cA,h} = (q_{\cA,h}^1,\ldots,q_{\cA,h}^k)$ be the sequence of first $k$ queries the algorithm makes if $h$ is the true hypothesis. $Q_{\cA,h}$ stands for the entire sequence until the algorithm stops with $V(Q_{\cA,h},h) = \{h\}$. Denote by $\circ$ the concatenation of two sequences. Let $\cost(Q,h)$ be the total cost of a set or sequence of queries $Q$ if the true hypothesis is $h$. For a set of points $X \subseteq \cX$ and a hypothesis $h\in\cH$, let $V(X,h)$ be the set of hypotheses that are consistent with the labeling of $h$ on $X$, that is $V(X,h) = \{g \in \cH \mid  \forall x \in S, g(x) = h(x)\}$.

By definition, there exists an algorithm $\cA$ such that for any $h \in \cH$, 
$\cost(Q_{\cA,h}) \leq \caud$. Denote $\caud$ by $\OPT$ for brevity.
Now, consider a greedy algorithm $\cB$. If $h$ is the true hypothesis then after $k$ queries, the version space is $V(Q^k_{\cB,h},h)$. Consider running $\cA$ after executing $Q^k_{\cB,h}$. Let the hypothesis $\bar{h} \in V(Q^k_{\cB,h},h)$ be such that for every length of sub-sequence $n \leq |Q_{\cA,\bar{h}}|$, and for 
every $y \in \cY$, 
\begin{equation}\label{eq:hbarmin}
\frac{|V(Q^k_{\cB,h} \circ Q^{n-1}_{\cA,\bar{h}},\bar{h}) \setminus V(Q^k_{\cB,h} \circ Q^{n}_{\cA,\bar{h}},\bar{h})|}{\cost(q^n_{\cA,\bar{h}},\bar{h}(q^n_{\cA,\bar{h}}))} \leq \frac{|V(Q^k_{\cB,h} \circ Q^{n-1}_{\cA,\bar{h}},g) \setminus V(Q^k_{\cB,h} \circ Q^n_{\cA,\bar{h}},g)|}{\cost(q^n_{\cA,\bar{h}},y)},
\end{equation}
where $g$ is equal to $\bar{h}$ on $Q^k_{\cB,h} \circ Q^{n-1}_{\cA,\bar{h}}$ but labels $q^n_{\cA,\bar{h}}$ differently.
Such a hypothesis clearly exists if there are two possible labels: it can be found by selecting, at each iteration, the hypothesis with the label for $q^n_{\cA,\bar{h}}$ that would incur the smaller ratio.

By the definition of $\OPT$,
\[
\sum_{n=1}^{|Q_{\cA,\bar{h}}|} \cost(q^n_{\cA,\bar{h}},\bar{h}) = \cost(Q_{\cA,\bar{h}})\leq \OPT.
\]
Also $|V(Q^k_{\cB,h}\circ Q_{\cA,\bar{h}},\bar{h})| = 1$. Therefore
\begin{align*}
\sum_{n=1}^{|Q_{\cA,\bar{h}}|} |V(Q^k_{\cB,h} \circ Q^{n-1}_{\cA,\bar{h}},\bar{h}) \setminus V(Q^k_{\cB,h} \circ Q^{n}_{\cA,\bar{h}},\bar{h})| &= |V(Q^k_{\cB,h},\bar{h}) \setminus V(Q^k_{\cB,h}\circ Q_{\cA,\bar{h}},\bar{h})|\\
&= |V(Q^k_{\cB,h},\bar{h})| -1 \\&= |V(Q^k_{\cB,h},h)| -1.
\end{align*}
It follows that there exists at least one $n$ such that 
\[
\frac{|V(Q^k_{\cB,h} \circ Q^{n-1}_{\cA,\bar{h}},\bar{h}) \setminus V(Q^k_{\cB,h} \circ Q^{n}_{\cA,\bar{h}},\bar{h})|}{\cost(q^n_{\cA,\bar{h}},\bar{h})} \geq \frac{|V(Q^k_{\cB,h},h)|-1}{\OPT}.
\]
 Moreover, for this $n$, due to \eqref{eq:hbarmin}, 
\[
\min_{y \in \cY} \frac{|V(Q^k_{\cB,h} \circ Q^{n-1}_{\cA,\bar{h}},g) \setminus V(Q^k_{\cB,h} \circ Q^n_{\cA,\bar{h}},g)|}{\cost(q^n_{\cA,\bar{h}},y)} \geq \frac{|V(Q^k_{\cB,h},h)|-1}{\OPT}.
\]
It follows that
\[
\min_{y \in \cY} \frac{|V(Q^k_{\cB,h},g) \setminus V(Q^k_{\cB,h}\circ q^n_{\cA,\bar{h}},g)|}{\cost(q^n_{\cA,\bar{h}},y)} \geq \frac{|V(Q^k_{\cB,h},h)|-1}{\OPT}.
\]
Therefore, the query $q^{k+1}_{\cB,h}$, selected by the greedy algorithm at step $k+1$, satisfies 
\[
\frac{|V(Q^k_{\cB,h} ,h) \setminus V(Q^{k+1}_{\cB,h},h)|}{\cost(q^{k+1}_{\cB,h},h)} \geq \frac{|V(Q^k_{\cB,h},h)|-1}{\OPT}.
\]
It follows that $\cost(q^{k+1}_{\cB,h},h) \leq \OPT$. In addition, it follows that 
\begin{align*}
|V(Q^{k+1}_{\cB,h},h)|-1 &\leq (|V(Q^{k}_{\cB,h},h)|-1)(1 - \cost(q^{k+1}_{\cB,h},h)/\OPT) \\
&\leq (|V(Q^{k}_{\cB,h},h)|-1)\exp(-\cost(q^{k+1}_{\cB,h},h)/\OPT).
\end{align*}
This analysis holds for every length $k$ of a sub-sequence $Q_{\cB,h}$. Therefore by induction
\begin{align*}
|V(Q^{k}_{\cB,h},h)|-1 &\leq (|\cH|-1)\prod_{i=1}^k\exp(-\cost(q^{i}_{\cB,h},h)/\OPT) \\
&= (|\cH|-1)\exp(-\cost(Q^k_{\cB,h},h)/\OPT).
\end{align*}
$\cB$ terminates at the minimal $k$ such that $|V(Q^{k}_{\cB,h},h)|-1 < 1$. This holds for any $k$ such that $\exp(-\cost(Q^k_{\cB,h},h)/\OPT) < 1/(|\cH|-1)$, which means $\cost(Q^k_{\cB,h},h) > \ln(|\cH|-1)\cdot \OPT$. Let $k'$ be the minimal integer that satisfies this inequality. Then $\cost(Q^{k'-1}_{\cB,h},h) \leq \ln(|\cH|-1)\cdot \OPT$. Since $\cost(q^{k'}_{\cB,h},h) \leq \OPT$, it follows that  $\cost(Q^{k'}_{\cB,h},h) \leq (\ln(|\cH|-1)+1)\cdot \OPT$.
This analysis holds for any $h \in \cH$, thus the worst-case cost of the greedy algorithm is at most $(\ln(|\cH|-1)+1)\cdot \OPT$.

\end{proof}

\bibliographystyle{abbrvnat}
\bibliography{mybib}

\end{document}